\documentclass[twoside]{IEEEtran}

\usepackage{amsmath,amsfonts}
\usepackage{algorithmic}
\usepackage{algorithm}
\usepackage{array}
\usepackage[caption=false,font=normalsize,labelfont=sf,textfont=sf]{subfig}
\usepackage{textcomp}
\usepackage{stfloats}
\usepackage{url}
\usepackage{hyperref}
\usepackage{verbatim}
\usepackage{graphicx}
\usepackage{cite}
\usepackage{caption}
\usepackage{subcaption}
\usepackage{multicol}
\usepackage{siunitx}
\usepackage{booktabs}
\usepackage[normalem]{ulem}
\usepackage{enumerate}

\usepackage{amsthm}
\usepackage{bm}
\usepackage{calc}
\usepackage{enumerate}
\usepackage{multicol}

\usepackage[table, dvipsnames]{xcolor}

\definecolor{dunkelblau}{rgb}{0.0, 0.2314, 0.6196}%
\definecolor{hellblau}{rgb}{0.000, 0.7451, 1.0000}%
\definecolor{rot}{rgb}{0.6980,0.1333,0.1333}%
\definecolor{cardinal}{rgb}{0.6156, 0.1333, 0.2078}
\definecolor{palo}{rgb}{0, 0.4157, 0.3216}
\definecolor{lagunita}{rgb}{0, 0.4863, 0.5725}
\newcommand{\mycomment}[1]{}

\newcommand{\rev}[1]{\textcolor{black}{#1}}
\newlength{\mlLegendThickness}
\newlength{\mlLegendHeight}
\newcommand{\mlLineLegend}[1]{\mbox{\color{#1}
		\protect\rule[\mlLegendHeight]{3mm}{\mlLegendThickness}\hspace*{-1mm}
}}

\newtheorem{theorem}{Theorem}

\newtheorem{proposition}[theorem]{Proposition}

\begin{document}
\newcommand\blankfootnote[1]{%
  \let\thefootnote\relax\footnotetext{#1}%
  \let\val\thefootnote%
}
\blankfootnote{©2012 IEEE. Personal use of this material is permitted. Permission from IEEE must be obtained for all other uses, in any current or future media, including
reprinting/republishing this material for advertising or promotional purposes, creating new collective works, for resale or redistribution to servers or lists, or
reuse of any copyrighted component of this work in other works. DOI:10.1109/LRA.2026.3739023.}

\title{Infinite-Dimensional Closed-Loop Inverse Kinematics for Soft Robots via Neural Operators}

\author{Carina Veil$^{1}$, \IEEEmembership{Member,~IEEE,} Moritz Flaschel$^{2}$, Ellen Kuhl$^{1}$, and Cosimo Della Santina$^{3}$ \IEEEmembership{Member,~IEEE,}
    \thanks{Manuscript received: 5 March 2026; Revised:
    2 July 2026; Accepted: 10 September 2026.} %
    \thanks{This paper was recommended for publication by Editor Barbara Mazzolai upon evaluation of the Associate Editor and Reviewers comments.}
	\thanks{
    This work is supported by the European Research Council (ERC) under the European Union's Horizon Europe research and innovation program, project \emph{RIPLEY}, grant agreement no. 101165078, the ERC Advanced Grant 101141626 \emph{DISCOVER}, as well as by the NSF CMMI Award 2318188 \emph{Mechanics of Bioinspired Soft Slender Actuators}. Corresponding author: Carina Veil.}
    \thanks{$^{1}$C. Veil and E. Kuhl are with the Department of Mechanical Engineering, Stanford University, Stanford, CA 94305, USA. 
    {\tt\small \{cveil,ekuhl\}@stanford.edu}}
   \thanks{$^{2}$ M. Flaschel is with the Institute of Applied Mechanics, Friedrich-Alexander-Universität Erlangen–Nürnberg, 91058 Erlangen, Germany.
		{\tt\small moritz.flaschel@fau.de}}
     \thanks{$^{3}$C. Della Santina is with the Cognitive Robotics Department, Delft University of Technology, 2628 Delft, The Netherlands. 
     {\tt\small c.dellasantina@tudelft.nl}}
     
\thanks{Digital Object Identifier (DOI): see top of this page.}
}

\markboth{IEEE ROBOTICS AND AUTOMATION LETTERS. PREPRINT VERSION. ACCEPTED SEPTEMBER 2026}%
{Veil \MakeLowercase{\textit{et al.}}: Infinite-Dimensional Closed-Loop Inverse Kinematics for Soft Robots via Neural Operators}


\maketitle

\begin{abstract}
For fully actuated rigid robots, kinematic inversion is a purely geometric problem, efficiently solved by closed-loop inverse kinematics (CLIK) schemes that compute joint \rev{velocities} to position the robot body in space.
For underactuated soft robots, however, not all configurations are attainable through control action, making kinematic inversion extremely \rev{ill-posed}. 
Extensions of CLIK address this by introducing end-to-end mappings from actuation to task space for the controller to operate on, but typically assume finite dimensions of the underlying virtual configuration space.
In this work, we formulate CLIK in the infinite-dimensional domain to reason about the entire soft robot shape while solving tasks, \rev{thereby skipping intermediate approximation steps.} We do this by composing an actuation-to-shape map with a shape-to-task map, deriving the differential end-to-end kinematics via an infinite-dimensional chain rule, and thereby obtaining a Jacobian-based CLIK algorithm. 
\rev{Since complex soft-robot models are rarely available in closed form, and neither the actuation-to-shape mapping nor its Jacobian can typically be obtained analytically,}
we propose to learn it using differentiable neural operator networks.
We first present an analytical study on a constant-curvature segment, and then apply the neural version of the algorithm to a three-fiber soft robotic arm whose underlying model relies on morphoelasticity and active filament theory.
\end{abstract}

\section{Introduction}
\IEEEPARstart{K}inematic inversion refers to finding a feasible robot configuration such that the robot assumes a desired positioning of one or more of its parts in space. For a rigid manipulator, this translates to computing joint positions that result in a desired end-effector pose, given the geometric link parameters.
This formulation, however, assumes full actuation of rigid robots. 
When a robot is soft, it can bend and twist continuously rather than only at discrete joints. This flexibility enables navigation in complex environments and safe interaction with delicate objects, but it also makes control, even basic kinematic inversion, challenging.

For \emph{fully actuated} systems, Jacobian-based closed-loop inverse kinematics ({CLIK}) schemes 
\rev{use the task-space error to compute joint-velocity updates via the inverse Jacobian, driving the joint configuration toward the desired task}
\cite{fiore2023convergence, kenwright2022real, colan2024variable}.
%
However, in \emph{under-actuated} systems with fewer degrees of actuation than degrees of freedom, such as soft robots, 
not all configurations are attainable through control action. Static kinematic constraints must be enforced explicitly, typically via a numerical optimization, which is then embedded, for example, in a model predictive control fashion \cite{wang2024hierarchical, votroubek2023globally}. A compact, closed-loop form similar to the fully actuated case only becomes possible once we introduce an end-to-end kinematic mapping from the actuator inputs to the task-space output \cite{webster2010design}. This mapping plays the role of a virtual joint-space model, like a controllable subspace, on which standard CLIK algorithms can then operate \cite{campisano2021closed, fang2022efficient, rogatinsky2023multifunctional}.
%
This so-called \emph{actuator-to-task closed-loop inverse kinematics} approach was formalized for underactuated mechanical systems in \cite{della2025pushing}. 
Yet, still assuming finite dimensions for the underlying virtual joint space, i.e., relying on the common soft robot approximation that the shape is linearly parametrizable with a set of basis functions, \rev{such as polynomial-curvature (PC) \cite{della2019control} or geometric variable-strain (GVS) \cite{renda2020geometric} formulations \cite{armanini2023soft}.}
\rev{While this introduces an intermediate discretization and approximation step, it also limits the tasks formulations.}

\begin{figure*}
    \centering
    \includegraphics[width=\linewidth]{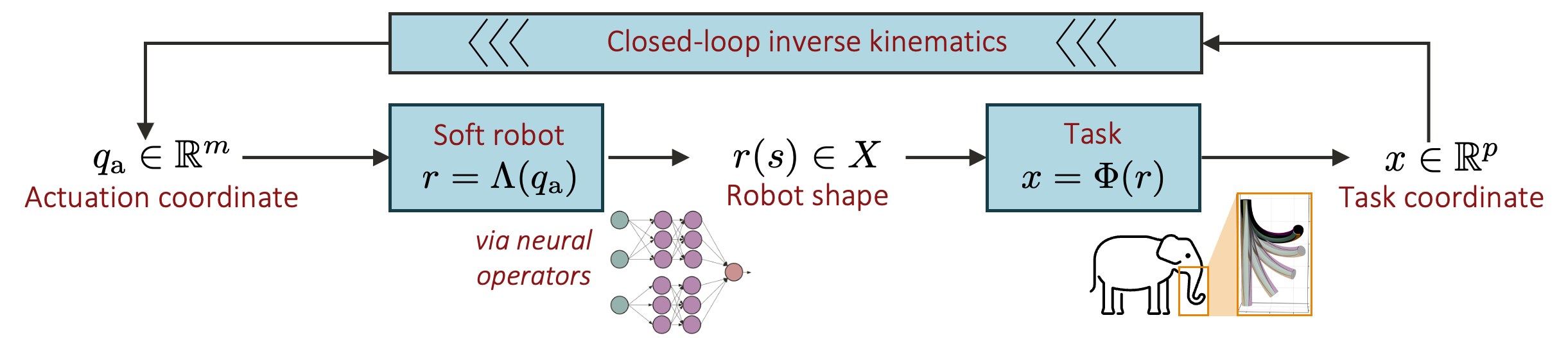}
    \caption{\textbf{(Neural) Infinite-dimensional closed-loop inverse kinematics (CLIK)}: We propose an extension of finite-dimensional CLIK that allows us to \rev{1) skip intermediate approximation steps, and 2)} reason on the whole soft robot shape while solving tasks. Since analytical expressions of the Jacobian are difficult to obtain in practice for infinite-dimensional soft robot models, we embed a \emph{neural} version of the algorithm that uses the neural operator of the robot learned from simulations.}
    \label{fig:abstract}
\end{figure*}

However, many soft-robot tasks require reasoning about the \textit{continuous body shape}, such as positioning any suitable point to a target or enforcing contact constraints along the body.
\rev{In this work, we extend CLIK to systems with infinite-dimensional shape spaces, while keeping the controller finite-dimensional through end-to-end mappings from finite actuators to finite task dimensions. This yields to main benefits:
\begin{itemize}
    \item \textbf{Avoiding intermediate approximations}: 
    By mapping actuation directly to task space, we skip discretizations or reduced-order approximations, such as PC or GVS. 
    \item \textbf{Reasoning on the shape}: We address tasks whose specifications depend on the entire body.
    While this does not control the full infinite-dimensional space, it enables task execution with shape-aware constraints.
\end{itemize}}
\rev{To construct the \emph{end-to-end forward kinematics} from actuation to task,} we define two mappings:
(i)~the mapping from actuators to shape, that is, the ``internal'' model of the (soft) robot,
(ii)~the forward kinematics from an infinite-dimensional shape space to a task,
\rev{and derive the corresponding differential kinematics for CLIK via an infinite-dimensional chain rule.}
Since soft robots’ actuator-to-shape mappings \rev{are rarely available in closed-form with explicit Jacobians, we propose to \emph{learn} a differentiable model of this map from simulation data using a \emph{neural operator network}.} Unlike traditional neural networks, they learn mappings between infinite-dimensional function spaces without fixed discretization, while providing well-defined gradients with respect to the finite-dimensional inputs \cite{anandkumar2020Neurala,lu2021Learning}. 
Neural operators have found wide interest in control and (soft) robotics already, for example as surrogate models for complex dynamics \cite{campbell2025Active, gao2024SimtoReal, ma2023Learning}, 
or to make computationally expensive control schemes either differentiable or real-time feasible \cite{bhan2023Neural, de2025deep, hu2025Neural}.
\rev{Here, we use them in precisely this spirit: to learn a differentiable, discretization-free actuation-to-shape map that makes end-to-end CLIK on continuous soft-robot shapes possible.}
\rev{Figure~\ref{fig:abstract} summarizes the concept. Notations are summarized in Table~\ref{tab:notation} at the end of the letter.}

\section{Preliminaries: Finite-Dimensional Closed-Loop Inverse Kinematics}\label{sec:prelim}

We will briefly review the finite-dimensional case and unify notations, but refer to \cite{siciliano1990closed, della2025pushing} for details.

\textbf{Kinematic inversion.}
Let $q\in\mathbb{R}^n$ be the configuration space of a mechanical system with $n$ degrees of freedom. Further, let $x \in\mathbb{R}^p$ be a $p$-dimensional task coordinate, commonly the end-effector position and orientation, such that the robot fulfilling the tasks is expressed through the forward kinematics
\begin{align}
    \Phi_\text{fin}: \mathbb{R}^n \rightarrow \mathbb{R}^p, \ q \mapsto x = \Phi_\text{fin}(q).\label{clik-1}
\end{align}
\rev{Given a desired task coordinate $\bar x$, the goal of kinematic inversion is to} find a configuration $\bar q$ such that $\bar x = \Phi_\text{fin}(\bar q)$.
Closed-loop inverse kinematics (CLIK)
solves the kinematic inversion in the form of dynamical system
\begin{align}
    \dot q = c_\text{IK}(q,\bar x) \qquad
    \text{s.t.} \lim_{t\to\infty}\Phi_\text{fin}(q(t))=\bar x,
\end{align}
\rev{with control law $c_\text{IK}(q,\bar x)$.}
A standard choice that ensures exponential convergence under mild assumptions is
\begin{align}
    \dot q = J_{\Phi_\text{fin}}(q)^{-1}K(\bar x - \Phi_\text{fin}(q)) ,\label{eq:clik-fully-actuated}
\end{align}
with positive definite gain matrix $K$ and Jacobian $J_{\Phi_\text{fin}}$, which is square and invertible for the fully actuated case. 

\textbf{CLIK for underactuated systems.}
For underactuated systems, we cannot control the $n$-dimensional configuration space, but we can extend the above to finding a control $c_\text{IK}$ that also respects static constraints, i.e.,
\begin{align}
    \dot q &= c_\text{IK}(q, \bar x) \notag \\
    &\text{s.t.} \lim_{t\to\infty}\Phi_\text{fin}(q(t)) = \bar x \text{ and } \lim_{t\to\infty}q(t)\in\mathcal{E},\label{eq:underactuated-clik-problem}
\end{align}
where $\mathcal{E}$ is the attainable equilibria set.
To realize this, we introduce an end-to-end mapping from actuators to task output that creates a virtual configuration subspace controllable by actuators on which CLIK can operate. Therefore, consider the finite-dimensional mapping
\begin{align}
\Lambda_\text{fin}:\mathbb{R}^m \to \mathbb{R}^n ,\ \ q_\text{a} \mapsto q,
\end{align}
from actuation coordinates $q_\text{a}$ with $m$ independent degrees of actuation to configuration coordinates $q$, such that the composed forward pseudo-kinematics 
are
\begin{align}
    (\Phi \circ \Lambda)_\text{fin}: \ \mathbb{R}^m  &\to \mathbb{R}^p, \notag \\  
    q_\text{a} &\mapsto   x = (\Phi \circ \Lambda)_\text{fin}(q_\text{a}).
\end{align}
Considering a square inversion ($m=p$) with Jacobian $J_{(\Phi \circ \Lambda)_\text{fin}}(q_\text{a})$, the standard gradient-based CLIK algorithm is
\begin{align}
    \dot q_\text{a} = \left(J_{(\Phi \circ \Lambda)_\text{fin}}(q_\text{a})\right)^{-1} K (\bar x - (\Phi \circ \Lambda)_\text{fin}(q_\text{a})).\label{eq:clik-finite}
\end{align}
A positive definite gain matrix $K \in \mathbb{R}^{m \times m}$ ensures exponential convergence analogous to the fully actuated case.
\rev{Note that more general redundant or underactuated formulations can be written using a pseudoinverse.}

\section{Infinite-Dimensional Closed-Loop Inverse Kinematics}\label{sec:inf-clik}
For the infinite-dimensional analogue of~\eqref{eq:clik-finite}, we introduce a (soft) robot model via the \textit{actuation-to-shape} mapping
\begin{align}
\Lambda:\mathbb{R}^m \to X,\ \ q_\text{a} \mapsto r=\Lambda(q_\text{a}),\label{eq:actuation-to-shape}
\end{align}
where $X=L^2([0,1];\mathbb{R}^3)$ is a Hilbert space of square-integrable curves with respect to the normalized spatial coordinate $s\in[0,1]$, and  equipped with the standard $L^2$-inner product. Furthermore, $r:[0,1] \to \mathbb{R}^3$ is a $C^1$ curve representing the soft robot shape in space\footnote{Normalized arc-length $s\in[0,1]$ is used without loss of generality via the standard reparameterization \rev{$s = \tilde{s}/L_\text{s}$, where $L_\text{s}$ is the robot length.} The formulation also holds in $\mathbb{R}^2$.}.
Note that we parametrize attainable shapes by actuator coordinates $q_\text{a}\in\mathbb{R}^m$, so the attainable-equilibria manifold forms an $m$-dimensional submanifold of the infinite-dimensional function space $X$.
Next, let $\Phi$ denote a $p$-dimensional task, i.e., the \textit{shape-to-task} mapping
\begin{align}
    \Phi: X \rightarrow \mathbb{R}^p, \quad r \mapsto x = \Phi(r).
\end{align}
such that the \textit{end-to-end} mapping from actuation to task on which CLIK operates is the composition 
\begin{align}
    \Phi \circ \Lambda: \mathbb{R}^m \rightarrow \mathbb{R}^p, \ \ q_\text{a} \mapsto x = (\Phi \circ \Lambda)(q_\text{a}).\label{eq:actuation-to-shape-infinite}
\end{align}

\subsection{Differential Kinematics}
To realize the infinite-dimensional CLIK version of~\eqref{eq:clik-finite}, we need the end-to-end Jacobian $J_{\Phi\circ\Lambda}(q_\text{a}) \in \mathbb{R}^{p\times m}$, which is obtained via the differentials of the component mappings $\Lambda$ and $\Phi$, connected by the chain rule.

\textbf{Actuation-to-shape.}
The Fr\'echet derivative of $\Lambda$ at $q_\text{a}$, $\Lambda' (q_\text{a}) $, is a bounded linear map 
\begin{align}
    &\Lambda' (q_\text{a}): \mathbb{R}^m\mapsto X, \notag \\
    &\quad \Lambda (q_\text{a} + \delta q_\text{a}) = \Lambda(q_\text{a}) + \Lambda'(q_\text{a}) [\delta q_\text{a}] + o(\Vert\delta q_\text{a}\Vert),
\end{align}
that maps small changes in actuators to small changes in shape. Since our actuation is finite dimensional, the Fr\'echet derivative reduces to a linear combination of partial derivatives in the function space,
\begin{align}
    \Lambda'(q_\text{a}) [\delta q_\text{a}] = \sum_{i=1}^{m} \delta q_{\text{a},i} \frac{\partial \Lambda(q_\text{a})}{\partial q_{\text{a},i}} ,\label{eq:actuation-to-shape-differential} \\
    \quad \frac{\partial \Lambda(q_\text{a})}{\partial q_{\text{a},i}}  =:\partial_{q_{\text{a},i}} \Lambda(q_\text{a}) \in X.\notag  
\end{align}

\begin{figure}[t!p]
    \centering
    \includegraphics[width=0.95\linewidth]{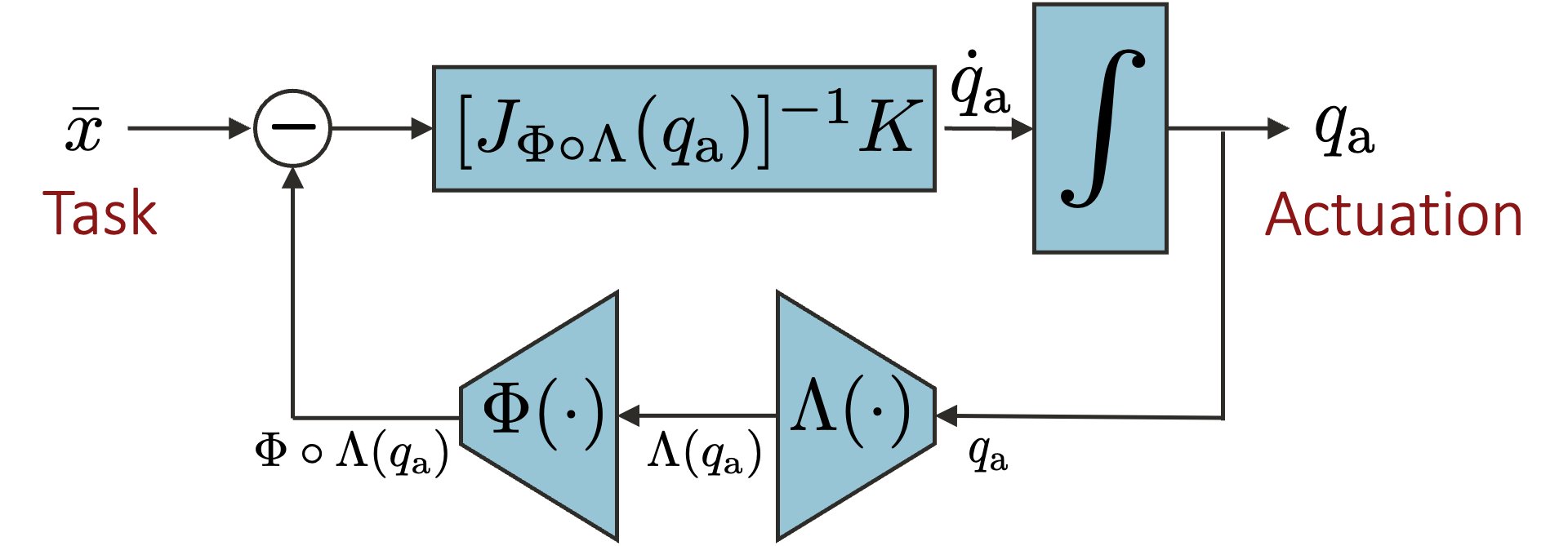}
    \caption{\rev{Block diagram of the infinite-dimensional closed-loop inverse kinematics~\eqref{eq:clik-infinite} in actuation space $q_\text{a}\in \mathbb{R}^m$, with a square actuation-to-task mapping. Trapezoids emphasize the change in dimensionality from input to output.}}
    \label{fig:clik-scheme}
\end{figure}

\textbf{Shape-to-task.}
For a perturbation $v:[0,\rev{1}]\mapsto \mathbb{R}^3$, define the perturbed shape $r_\varepsilon= r +\varepsilon v$ and the directional derivative 
\begin{align}
    \Phi'(r)[v] = \frac{d}{d \varepsilon}\Phi(r_\varepsilon)\big|_{\varepsilon=0}.
\end{align}
The directional derivative $\Phi'(r)[v]$ defines a bounded linear functional on the tangent space of admissible perturbations $v$. 
Endowing this space with a Hilbert structure (such as $X=L^2([0,1];\mathbb{R}^3)$) with inner product
\begin{align}
    \langle u,v \rangle _{L^2}=\int_0^1 u(z) \cdot v(z)dz, 
\end{align}
the Riesz representation theorem \cite{neuberger2009sobolev} ensures the existence of a unique element $\nabla_{X} \Phi(r) \in X $ such that
\begin{align}
    \Phi'(r) [v] = \langle \nabla_{X} \Phi(r) , v \rangle _{X}, \quad \forall v \in X, \label{eq:task-to-shape-differential}
\end{align}
i.e., the $L^2$ gradient of $\Phi$ at $r$.

\textbf{End-to-end.}
For the CLIK algorithm, we consider the differential end-to-end kinematics
\begin{align}
    \dot x = J_{\Phi \circ \Lambda}(q_\text{a})\dot q_\text{a}, \ J_{\Phi \circ \Lambda}\in \mathbb{R}^{p \times m}. \label{eq:end-to-end-differential}
\end{align}
For any direction $ \dot q_\text{a}\in\mathbb{R}^{m}$, by chain rule, the derivative of this composition is
\begin{align}
    (\Phi \circ \Lambda)'( q_\text{a}) [ \dot  q_\text{a}] = 
    \big( \Phi'(r)\big)
    \big(\Lambda'( q_\text{a})[\dot  q_\text{a}]\big).
\end{align}
Together with the task differential $\Phi'(r) [v]$~\eqref{eq:task-to-shape-differential}, this yields
\begin{align}
(\Phi \circ \Lambda)'( q_\text{a}) [ \dot  q_\text{a}]
&= \big\langle \nabla_{X} \Phi(r),\; \Lambda'( q_\text{a})[\dot q_\text{a}]\big\rangle_X \\
&= \big\langle \Lambda'( q_\text{a})^{*}\,\nabla_{X} \Phi(r),\; \dot q_\text{a}\big\rangle_{\mathbb{R}^m},
\end{align}
where $\Lambda'( q_\text{a})^*: X \rightarrow \mathbb{R}^m$ is the Hilbert adjoint. 
Comparing this with the differential end-to-end kinematics~\eqref{eq:end-to-end-differential}, we note that the Jacobian of the composition must be
\begin{align}
    J_{\Phi\circ\Lambda}( q_\text{a}) = \Lambda'( q_\text{a})^* \nabla_{X} \Phi(\Lambda( q_\text{a})) \in \mathbb{R}^{p\times m},
\end{align}
and its $i$-th column is calculated as
\begin{align}
J_{\Phi\circ\Lambda, :, i}(q_\text{a}) &= \langle \nabla_{X} \Phi(\Lambda( q_\text{a})), \partial_{ q_{\text{a},i}} \Lambda( q_\text{a})\rangle_{X}\\
&= \int_0^L \partial_{ q_{\text{a},i}} \Lambda( q_\text{a})(z) \cdot  \nabla_{X} \Phi(\Lambda( q_\text{a}))(z) \ dz,
\end{align}
for $i=1,2,..,m$.
When $p=m$ and $J_{\Phi\circ\Lambda}(q_\text{a})$ is nonsingular, following the usual steps for kinematic inversion yields the final CLIK algorithm~(cf. Figure~\ref{fig:clik-scheme})
\begin{subequations}
\begin{align}
    \dot q_\text{a} &= \left(J_{\Phi \circ \Lambda}(q_\text{a})\right)^{-1} K (\bar x - (\Phi \circ \Lambda)(q_\text{a})), \\
    r &= \bar r(q_\text{a}) 
\end{align} \label{eq:clik-infinite}
\end{subequations}
with positive definite gain matrix $K \in \mathbb{R}^{m \times m}$.


\begin{proposition}
    The closed-loop kinematic inversion~\eqref{eq:clik-infinite} solves~\eqref{eq:underactuated-clik-problem} exponentially fast $\forall K \succ 0$.
\end{proposition}

\begin{proof}
    First, we formulate the differential kinematics of the end-to-end mapping~\eqref{eq:end-to-end-differential}, and plug in the proposed controller~\eqref{eq:clik-infinite}
    \begin{align}
        \dot x &= J_{\Phi \circ \Lambda}(q_\text{a})\dot q_\text{a} \\
        &= J_{\Phi \circ \Lambda}(q_\text{a})
        \left(J_{\Phi \circ \Lambda}(q_\text{a})\right)^{-1} K (\bar x - (\Phi \circ \Lambda)(q_\text{a}))\\
        &= K (\bar x - (\Phi \circ \Lambda)(q_\text{a})).
    \end{align}
    We observe that the resulting error dynamics $\dot{\bar{x}} - \dot x=-K(\bar x - x)$ is linear and converges exponentially fast to zero with the rate defined by $K$. Consequently, the steady state of~\eqref{eq:clik-infinite} fulfills $\lim_{t\to\infty}\Phi(q(t))=\bar x$.
\end{proof}
\rev{\emph{Remark.}
Proposition~1 assumes an exact, nonsingular Jacobian. In practical implementations, numerical approximation and near-singular configurations may affect convergence; damping or regularization can be introduced in the usual way.}

\subsection{Representative Tasks and Resulting Jacobians}
We can think of many functionals $\Phi: X \to \mathbb{R}^p$ in the task space of a robot. 
The most common examples involve reaching a target or a prescribed distance from it for a specific point along the robot’s body, typically the end-effector for rigid robots.
While such tasks are equally relevant for soft robots, their continuous nature invites an adaptation of this formulation: rather than requiring the end-effector to reach the target, we may instead define the task such that the point along the robot’s body closest to the target reaches it (cf. Figure~\ref{fig:tasks}).

\textbf{Positioning tasks.}
Reaching a target point $x_0$ 
at a fixed centerline coordinate $\bar s \in [0,1]$ of the robot centerline $r$ translates to
\begin{align}
    \Phi_{\text{pos}}^{\text{fixed}}(r) &= r(\bar s) - x_0. 
\end{align}
However, our infinite-dimensional CLIK shines for the more expressive \emph{closest-point} task, which automatically identifies the best location along the entire robot shape to reach the target. \rev{Such a formulation is particularly useful when contact, reachability, or avoidance constraints should be satisfied by the best-suited point along the continuum body rather than by a fixed endpoint.}
This is formulated as
\begin{subequations}
\begin{align}
    \Phi_\text{pos}^\text{opt}(r) &= r(s_*) - x_0, \\
    \text{where } s_* &= \text{argmin}_{s\in[0,1]} \frac{1}{2}\|r(s) - x_0\|^2,\label{eq:closest-point}
\end{align}
\end{subequations}
meaning that the closest centerline coordinate $s_*$ to the target is computed via gradient descent in each step.

The $L^2$-gradients of the positioning tasks are $\nabla_{L^2} \Phi_{\text{pos}}^{\text{fixed}}(r) = \delta(s- \bar s) \in X$, and $\nabla_{L^2} \Phi_\text{pos}^\text{opt}(r) = \delta(s-s_*) \in X$, such that the Jacobians of the end-to-end mapping $J_{\Phi\circ\Lambda} \in \mathbb{R}^{3\times m}$ have $m$ columns
\begin{align}
    J_{(\Phi_\text{pos}^\text{fixed}\circ\Lambda):,i}(q_\text{a}) = \partial_{q_{\text{a},i}} \Lambda(q_\text{a})(\bar s),\label{eq:jac-comp-pos-fixed}\\
    J_{(\Phi_\text{pos}^\text{opt}\circ\Lambda):,i} (q_\text{a}) = \partial_{q_{\text{a},i}} \Lambda(q_\text{a})(s_*),\label{eq:jac-comp-pos-opt}
\end{align}
i.e., only the partial derivative of the actuation-to-shape mapping at the respective $s$-coordinate remains.

\textbf{Distance tasks.}
Scalar distance tasks follow the same structure. The fixed-coordinate variant minimizes squared distance from a prescribed centerline coordinate $\bar s \in [0,1]$
\begin{align}
    \Phi_\text{dist}^{\text{fixed}}(r) &= \frac{1}{2} \Vert r(\bar s) - x_0\Vert ^2.
\end{align}
Again, infinite-dimensional CLIK excels when minimizing the distance from the closest point on the centerline
\begin{align}
    \Phi_\text{dist}^{\text{opt}}(r) &= \min_{s \in [0,1]} \frac{1}{2} \Vert r(s) - x_0\Vert ^2,
\end{align}
where, the closest centerline coordinate $s_*$ is computed via gradient descent.
Assuming uniqueness and nondegeneracy of the minimizer (Infinite-dimensional Danskin's Theorem, cf. Theorem 4.13 in \cite{bonnans2013perturbation}), its directional derivative follows~\eqref{eq:task-to-shape-differential} with $L^2$-gradients $\nabla_{L^2} \Phi_{\text{dist}}^{\text{fixed}}(r) = (r(s_*)-x_0)\delta(s- s_*) \in X$, and $\nabla_{L^2} \Phi_\text{dist}^\text{opt}(r) = (r(s_*)-x_0)\delta(s-\bar s) \in X$, such that the row Jacobians of the end-to-end mapping $J_{\Phi\circ\Lambda} \in \mathbb{R}^{1\times m}$ have $m$ scalar entries
\begin{align}
    J_{(\Phi_\text{dist}^\text{fixed}\circ\Lambda),i} (q_\text{a}) &= \left[\partial_{q_{\text{a},i}} \Lambda(q_\text{a})(\bar s)\right]^\top(r(\bar s)-x_0), \\
    J_{(\Phi_\text{dist}^\text{opt}\circ\Lambda),i} (q_\text{a}) &= \left[\partial_{q_{\text{a},i}} \Lambda(q_\text{a})(s_*)\right]^\top(r(s_*)-x_0).
\end{align}

\rev{\textbf{Distributed task.}
To make even better use of the shape formulation and avoid the collapse to a Dirac contribution at a single centerline coordinate, consider the distributed whole-body task
\begin{align}
    \Phi_{\text{wavg}}(r) &= \frac{1}{W}\int_0^L w(s)r(s)\,ds \in \mathbb{R}^3,\label{eq:phi-wavg}
\end{align}
with nonnegative weight $w(s)$ and normalization $W = \int_0^L w(s) ds$, so that the control objective is to move the weighted average position of the entire backbone $r$ to a desired target $x_0 \in \mathbb{R}^3$. The desired target $x_0$ is not part of the task functional, but defines the tracking residual used by CLIK. 
Unlike point-wise tasks, the gradient density is distributed along the body according to $w(s)$, so the resulting Jacobian reflects the contribution of the entire centerline,
\begin{align}
    J_{(\Phi_{\mathrm{wavg}}\circ\Lambda):,i}(q_\mathrm{a})
    &= \frac{1}{W}\int_0^L w(s)\,\partial_{q_{\mathrm{a},i}}\Lambda(q_\mathrm{a})(s)\,ds.
\end{align}
In the special case $w(s)\equiv 1$, this reduces to the plain geometric centroid with
\begin{align}
    \Phi_{\text{cen}}(r) &= \frac{1}{L}\int_0^L r(s)\,ds ,\label{eq:phi-centroid}\\ J_{(\Phi_{\mathrm{cen}}\circ\Lambda):,i}(q_\mathrm{a})
    &= \frac{1}{L}\int_0^L \partial_{q_{\mathrm{a},i}}\Lambda(q_\mathrm{a})(s).
\end{align}
}

\begin{figure}
    \centering
    \includegraphics[width=0.72\linewidth]{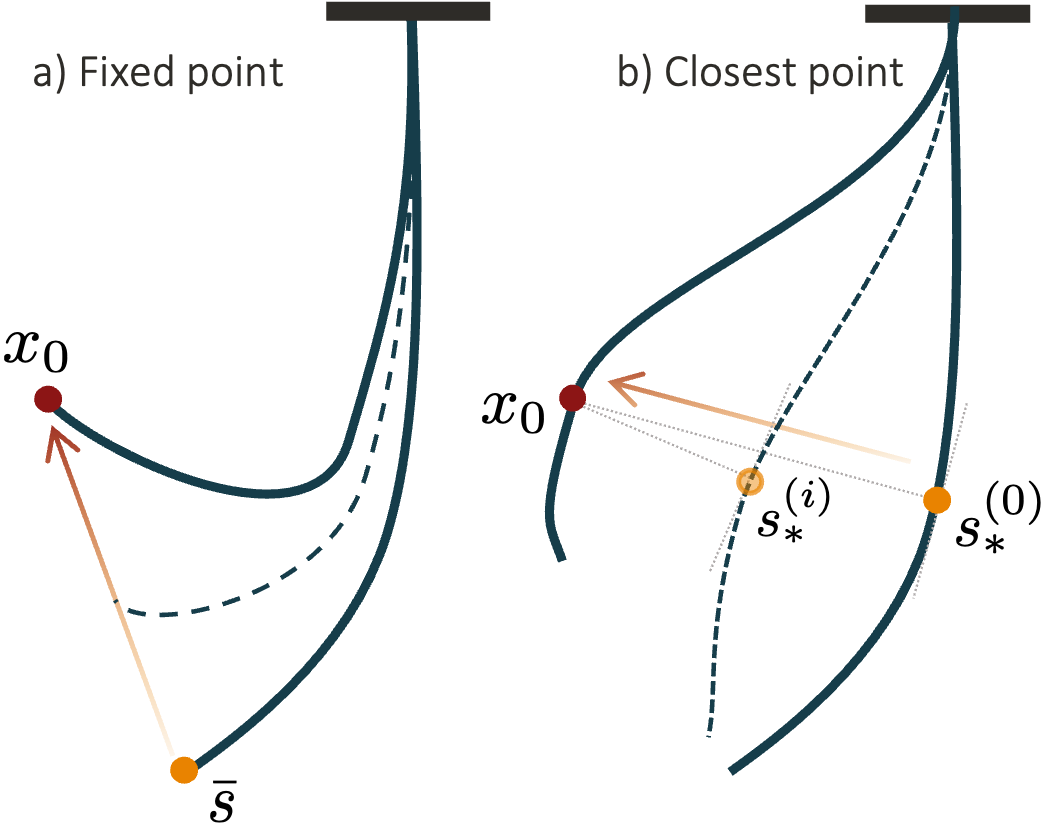}
    \caption{While standard tasks usually focus to position the end-effector $r(\bar s=1)$, the continuum nature of the soft robot invites for just reaching a target point with the closest point $r(s_*)$ on its backbone, where $s_*$ is adapted in each iteration.}
    \label{fig:tasks}
\end{figure}

\subsection{Toy Example: Constant Curvature Segment}
Let us now consider an analytical example to demonstrate the infinite-dimensional CLIK algorithm. For the sake of simplicity, and almost w.l.o.g., assume a planar, inextensible robot such that the only strain is the curvature $\kappa(s)$, where $s \in [0,1]$ is a normalized arc‑length coordinate and the physical arc length is $L$. This leads to infinite-dimensional mapping $\Lambda_\kappa$
that maps curvature functions to a backbone pose ($x,y,\alpha$), 
more precisely,
\begin{align}
    \Lambda_\kappa(\kappa)(s) = \begin{bmatrix}
        x(s) \\ y(s) \\ \alpha(s)
    \end{bmatrix}
    =\begin{bmatrix}
        L\int_0^s\cos(\alpha(v))dv\\
        L\int_0^s\sin(\alpha(v))dv \\
                \int_0^s\kappa(v)dv
    \end{bmatrix}.
\end{align}
Assuming a constant curvature $q$ along the segment such that $\kappa(s) \approx q$, and focusing on the centerline position $r(s)=[x(s),y(s)]$ without orientation, this simplifies to
\begin{align}
     \Lambda_\text{CC}(q)(s)=
     r(s)
    =\begin{bmatrix}
         L \frac{\sin(sq)}{q} \\
         L \frac{1-\cos(sq)}{q}
    \end{bmatrix},
\end{align}
with Jacobian
\begin{align}
    J_\text{CC}(q) 
    = 
    L 
    \begin{bmatrix}
    \frac{sq\cos(sq)-\sin(sq)}{q^2}   &
    \frac{(\cos(sq)-1)+sq\sin(sq)}{q^2} 
    \end{bmatrix}^\top.
\end{align}
Since the constant curvature segment has only one degree of actuation, namely this curvature itself, i.e., $q_\text{a}=q$, we require a one-dimensional task such as the scalar distance task $\Phi_\text{dist}^{\text{fixed}}$.
The scalar Jacobian of the composition is then
\begin{align}
    J_{\Phi_{\text{dist}}^{\text{fixed}} \circ \Lambda_{\text{CC}}} (q) = J_\text{CC}(q)(\bar s)(r(\bar s) - x_0).
\end{align}
Reaching a point translates to eliminating the distance, i.e., $\bar x = \Phi_\text{dist}^{\text{fixed}}(\bar r) = 0$, such that the CLIK control is
\begin{align}
    \dot q = - K \left[J_{(\Phi_{\text{dist}}^{\text{fixed}} \circ \Lambda_{\text{CC}})}(q)\right]^{-1}\Phi_{\text{dist}}^{\text{fixed}}(\Lambda_\text{CC}(q)),
\end{align}
with scalar \rev{gain} $K>0$. The same holds for the closest-point version of the task $\Phi_\text{dist}^{\text{opt}}$ with $s_*$ instead of $\bar s$.

Figure~\ref{fig:cc-simulations} shows simulations of such a constant curvature segment reaching a target point at a fixed backbone coordinate (the tip position in this case), as well as the closest-point task where any point on the robot should reach the target as fast as possible. In both scenarios, the scalar \rev{gain} is $K=10$ and we observe exponential convergence. For the closest-point task, the optimal coordinate $s_*$ is continuously updated to the point nearest to the target, \rev{resulting in the $s_*$ trajectory (orange).}

\begin{figure}[t!p]
    \centering
    \includegraphics[width=0.85\columnwidth]{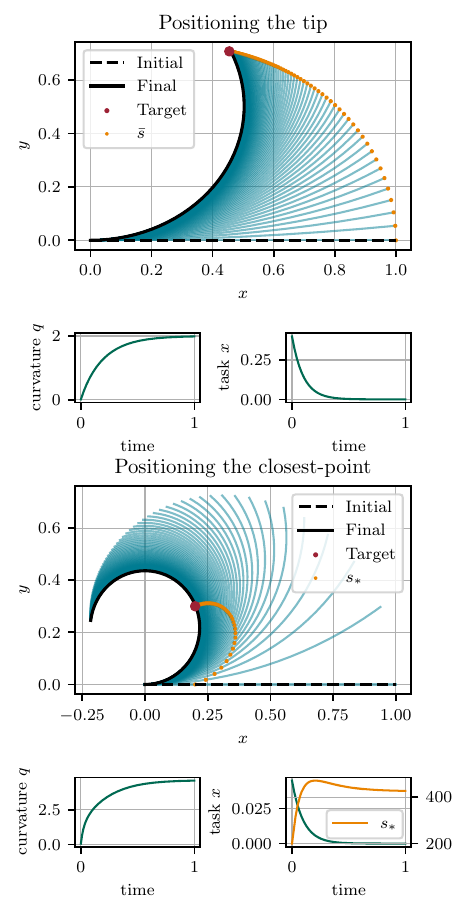}
    \caption{Constant curvature segment: robot evolution and \rev{corresponding discrete CLIK iterations (\(1000\) steps with \(dt=0.001\)) of the actuation variable \(q\) and the task variable \(x\)} for the point-reaching task $\Phi_{\text{pos}}^{\text{fixed}}$ (top) and the closest-point task $\Phi_{\text{opt}}^{\text{fixed}}$ (bottom).}
    \label{fig:cc-simulations}
\end{figure}

\section{Neural Closed-Loop Inverse Kinematics}\label{sec:neural-inf-clik}
While the previous example demonstrates how the infinite-dimensional CLIK works and what possibilities we have when being able to reason on the whole soft body shape while solving tasks, 
we recognize that the first of the two maps $\Lambda$ is very hard to obtain in closed form in
practice, especially once more complex soft robots are involved. This is why we propose to learn the soft robot model operator of the actuation-to-shape mapping $\Lambda$~\eqref{eq:actuation-to-shape}, denoted as $\mathcal{G}_\Lambda$, via \emph{neural operator networks}.

\subsection{Neural Operator Networks and Their Gradients}
The universal operator theorem \cite{chen1995Universal} shows that a neural network with one hidden layer can approximate any continuous nonlinear operators.
Extending this, the \emph{DeepONet} \cite{lu2021Learning} architecture allows for learning operators of the form $\mathcal{G}:u\mapsto \mathcal{G}(u)$, where both $u$ and $\mathcal{G}(u)$ are functions defined on (infinite-dimensional) spaces.
For any point $s$ in the domain of $\mathcal{G}(u)$, the value $\mathcal{G}(u)(s)$ is a real-valued vector in $\mathbb{R}^d$. To learn said operator, \emph{DeepONet} takes as inputs the sampled values $[u(x_1), u(x_2),..., u(x_m)]$ at a finite number of fixed locations ${x_1, x_2,..., x_m}$, together with the evaluation points $s$.
The architecture consists of two sub-networks, the branch network processing the input function, and the trunk network processing the evaluations points,
\begin{equation}
\begin{aligned}
\mathcal{G}^{\text{branch}}&: \mathbb{R}^m \to \mathbb{R}^{\rev{\omega} \times d}, \quad u \mapsto a_{\text{br}}, \\
\mathcal{G}^{\text{trunk}}&:\mathbb{R} \to \mathbb{R}^{\rev{\omega}\times d}, \quad  s \mapsto a_{\text{tr}}, \\
\end{aligned}
\end{equation}
where $\rev{\omega}$ is a hyperparameter controlling the \rev{dimension} of the latent representation. 
The final prediction is computed as the inner product of the outputs of the branch and trunk networks over the latent dimension
\begin{align}
\mathcal{G}(u)(s) &= a_{\text{br}} \odot a_{\text{tr}}, 
\end{align}
where $\odot$ denotes the element-wise multiplication over the latent dimension.
%
Once trained, the \emph{DeepONet} realization is a differentiable mapping with respect to all its inputs, so gradients with respect to both the input function samples and the evaluation coordinate $s$ are obtained directly by automatic differentiation. 
%
In practice, the branch network is trained on finitely many function samples, while at inference time the trunk network can be evaluated at arbitrarily many, arbitrarily fine locations $s$. This decouples the discretization of the input function from the resolution of the output, which is a key advantage of neural operators over discretization-based neural networks, and makes it particularly powerful for infinite-dimensional CLIK where we want to reason over the continuous robot shape and in finer discretizations.

\subsection{Example: Three-fiber Soft Robotic Arm}
To demonstrate our neural infinite-dimensional CLIK algorithm, we consider a soft robotic arm initially designed with three contractable fibers \rev{realized using liquid crystal elastomers (LCE) \cite{leanza2024Elephant},} inspired by the elephant trunk musculature. It is depicted in Figure~\ref{fig:trunk}. By design, this structure is soft and slender, and can grow, shrink, or actively remodel its intrinsic shape while also deforming elastically -- which is why it is modelled as an \textit{active, morphoelastic filament} \cite{kaczmarski2022Active}. This model has been previously validated in experiments \cite{leanza2024Elephant,kaczmarski2024Minimal}, successfully used for path-planning algorithms \cite{veil2026shape}, and shall now serve as example for the neural operator implementation of infinite-dimensional CLIK. 

\rev{The actuation coordinates $q_\text{a}\in\mathbb{R}^3$ represent the electrical currents driving the three LCE fibers, resulting the robot to adopt shape $r(s)$.}
Its forward kinematics are governed by 
\begin{subequations}  \label{eq:fk}
\begin{align}
    r'(s) &= \zeta d_3,\\
    d_i'(s) &= \hat \zeta u \times d_i, \quad i=\{1,2,3\},\\
    0 &= \frac{\partial n}{\partial s} + \hat \zeta f , \\
    0 &= \frac{\partial m}{\partial s} + \frac{\partial r}{\partial s} \times n + \hat \zeta l.
\end{align}\label{eq:bvp}%
\end{subequations}
with centerline $r: [0, 1] \rightarrow \mathbb{R}^3$, normalized spatial coordinate $s\in[0,1]$, orthonormal director frame $D = \{d_1, d_2, d_3\}$, as well as extensions $\zeta$, curvatures $u$, 
internal force $n$, external force $f$, internal momentum $m$.
The mapping from actuation $q_\text{a}\in \mathbb{R}^3$ to shape $r(s)$ is mediated by the intrinsic curvature $\hat u$ and extension $\hat \zeta$, both of which are functions of the fiber activations, hence, the actuation, i.e., $\hat u(q_\text{a})$ and $\hat \zeta(q_\text{a})$. 
Details are found in \cite{kaczmarski2022Active, kaczmarski2024Minimal}.

\subsection{Learning the Model with a Neural Operator}
\rev{The learned operator maps the three-dimensional actuation vector $q_\text{a}$ to the continuous centerline shape $r(s)$.}
\rev{Hence,} in our application, the architecture of the operator network is simplified compared to \emph{DeepONet} \cite{lu2021Learning}, since the actuation coordinates $q_\text{a}\in\mathbb{R}^m$ are finite-dimensional. Consequently, the branch net reduces to a standard feedforward network, while the trunk net only encodes the spatial variable $s$. We still retain the key property that the operator is learned over actuation and space, so we can evaluate the learned actuation-to-shape mapping $\mathcal{G}_\Lambda$ at arbitrarily many points $s$ and differentiate with respect to $q_\text{a}$.
\rev{At this point, we note that an alternative strategy would be to represent the output function in an a priori chosen finite basis and use a conventional neural network to map the actuation coordinates to the corresponding basis coefficients. Such an approach would likewise satisfy the desired differentiability properties. However, the chosen basis may not possess the approximation properties associated with \emph{DeepONets} \cite{chen1995Universal}. Furthermore, our architecture does not assume a predefined basis. Instead, it learns a representation of the output function directly from data.}


\begin{figure}[t!p]
    \centering
    \includegraphics[width=0.9\columnwidth]{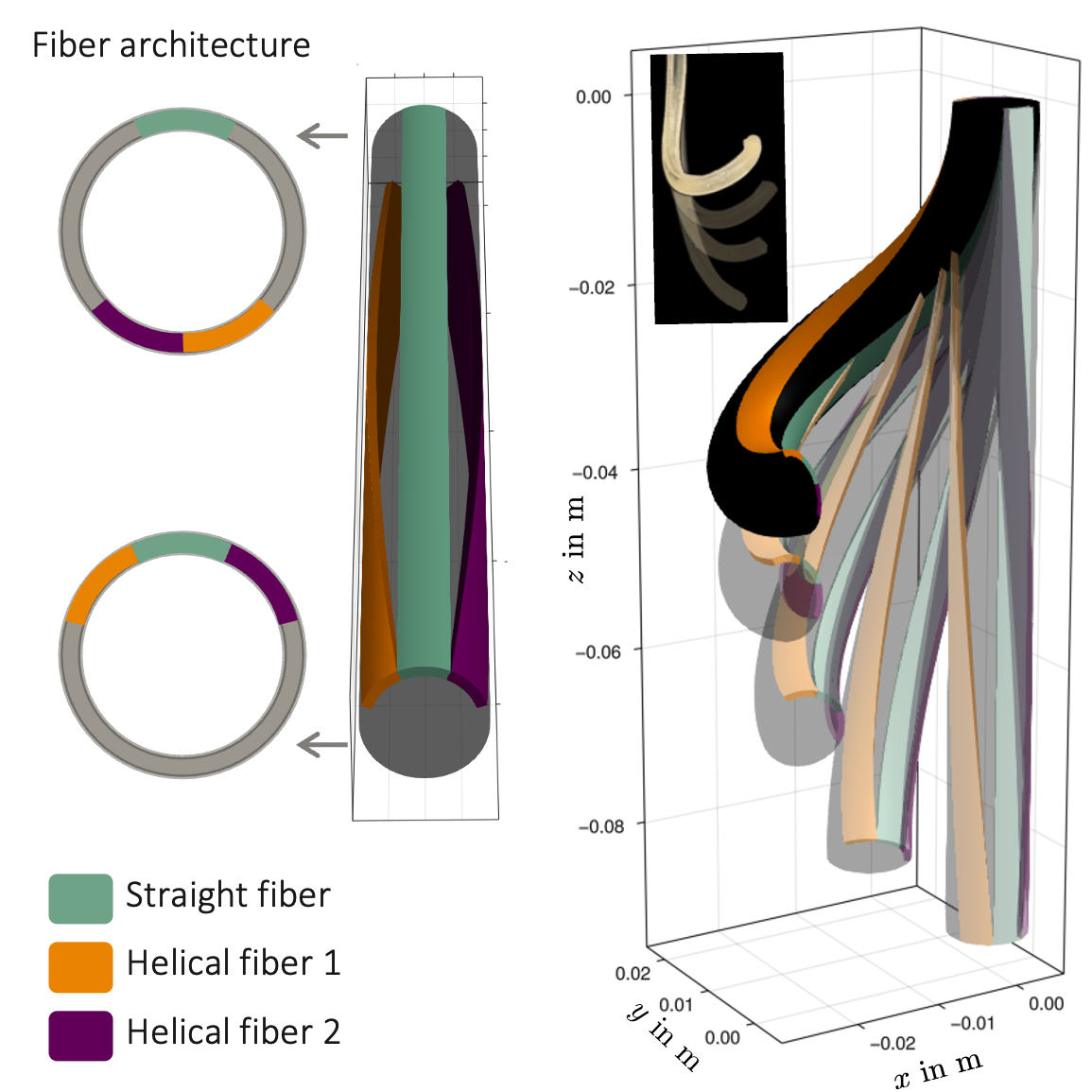}
    \caption{Three-fiber soft robotic arm \cite{kaczmarski2024Minimal, leanza2024Elephant}: The orientation of the fibers is inspired by the muscular structure of the elephant trunk, with one straight fiber and a helical fiber pair, enabling continuous deformations and optimal reachability.}
    \label{fig:trunk}
\end{figure}
 
\textbf{Data.}
For data generation, we solve the boundary value problem~\eqref{eq:bvp} $N=1,000,000$ times to obtain samples of the actuation-to-shape mapping for random activations within the physical constraints of the system $q_{\text{a},i}\in[-1.67,0]$, $i=1,2,3$. Note that $q_{\text{a}}=0$ corresponds to the steady state of the system in which the soft robot hangs downward, and negative actuation leads to fiber contraction relative to this state. Histograms of the activations show that the entire actuation space is sufficiently represented in the dataset. The spatial coordinate $s$ is discretized with $n_\text{s}=100$ points in $[0,1]$.
We use the \textit{Julia} package developed for the active filament theory \cite{kaczmarski2022Active} and the associated parameters of the minimal three fiber prototype \cite{leanza2024Elephant}. 
We store each solution as a discrete centerline in a matrix $R \in \mathbb{R}^{n_s \times 3}$ with its respective $(x,y,z)$-coordinates at the $n_s$ locations. 

\begin{figure*}[t!p]
    \includegraphics[width=\textwidth]{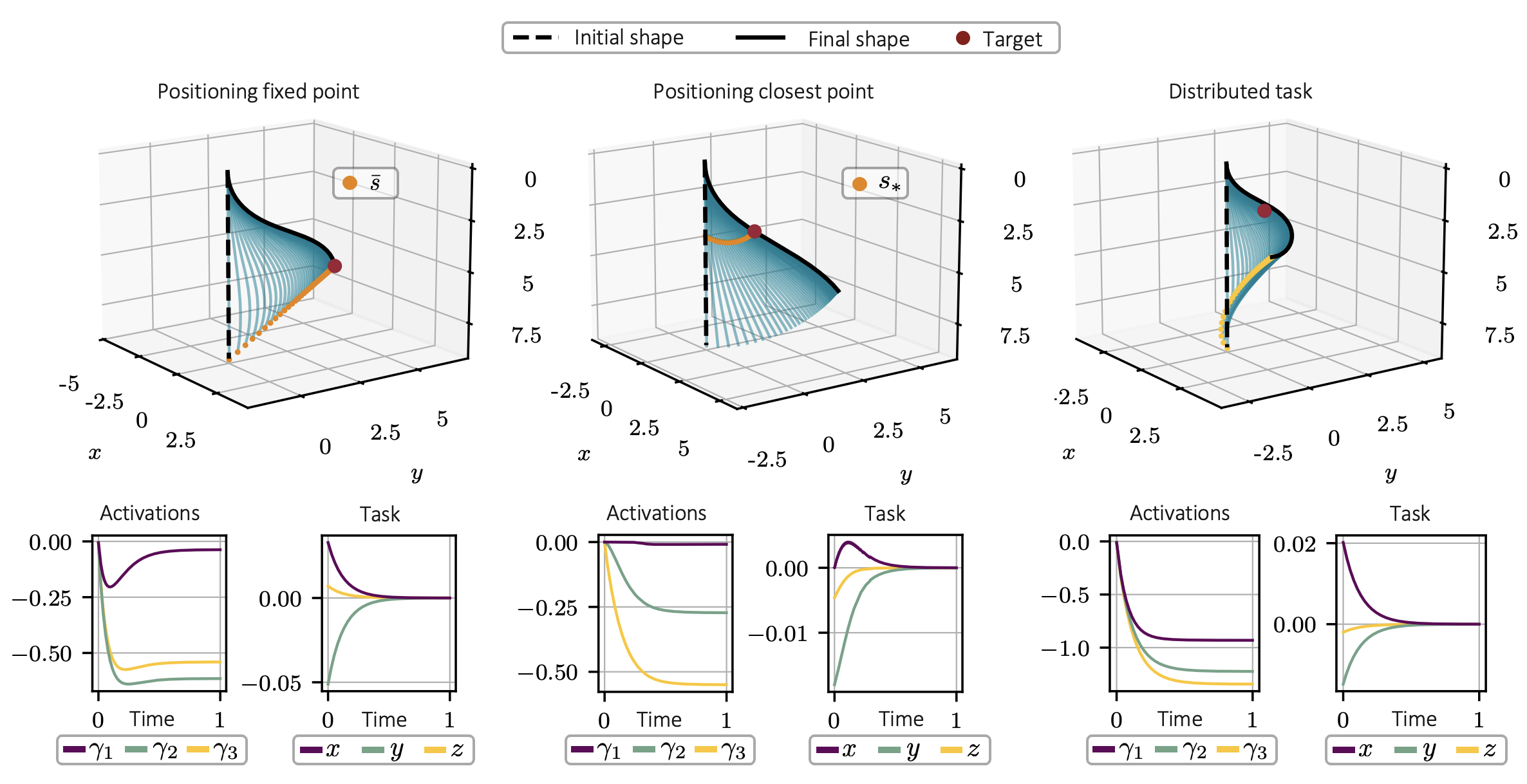}
    \caption{\rev{Three-fiber soft robotic arm. Robot evolution (\mlLineLegend{lagunita}) and time evolutions of the actuation variable (\emph{activations} $\gamma_i$) and the task variable $x$ for the fixed point-reaching task $\Phi_\text{pos}^\text{fixed}$ (left), closest-point task $\Phi_\text{pos}^\text{opt}$ (center), and distributed task $\Phi_\text{cen}$ (right). For the robot evolution, we plot every 10th shape.
The axes of the shape plot are multiplied by 100 for better readability. For the distributed task, the tip position is highlighted in yellow to better track the trajectory}}
\label{fig:neural-clik}
\end{figure*}

\textbf{Training process.}
Of the 1,000,000 samples, 200,000 are held out as a test set. From the remaining data, 640,000 samples are used for training and 160,000 for validation. We train the operator network for 500 epochs with the Adam optimizer, \texttt{tanh} activations, batch size 32, and an exponentially decaying learning rate, using the validation loss to select the best model checkpoint (typically reached after 400-450 epochs). On the held-out test set, the trained model achieves a mean squared error of $1.38\times 10^{-10}$ and an $L^2$ relative error of $6.08\times 10^{-4}$.
The architecture and test performance are summarized in Table~\ref{tab:hyperparameters}.
\rev{Note that smaller datasets can achieve accuracies with MSE performance reaching $10^{-8}$ at significantly decreased computational cost. However, given the nature of this work, the dataset size was chosen pragmatically to obtain the broadest coverage of the actuation space, rather than focusing on efficiency.
Systematic ablation studies are therefore left for future work.}

\begin{table}[h]
\centering
\begin{tabular}{ll}
\toprule
\textbf{Architecture and Training}  \\
\toprule
Optimizer               & \texttt{Adam} \\
Activation              & \texttt{tanh} \\
Learning rate           & \texttt{Exponential decay} \\
Epochs                  & \texttt{500} \\
Batch size              & \texttt{32} \\
Branch layers           & \texttt{[3, 64, 64, 64, 192]} \\
Trunk layers            & \texttt{[1, 64, 64, 64, 192]} \\ 
\bottomrule
\textbf{Test Performance} \\
\toprule
MSE                 & $1.38\times 10^{-10}$  \\
L2 relative error   & $6.08\times 10^{-4}$  \\ 
\bottomrule
\end{tabular}
\caption{Neural operator hyperparameters and performance.}
\label{tab:hyperparameters}
\end{table}

\subsection{Composed Jacobian and Neural CLIK}
\rev{To validate our approach, we consider the tasks $\Phi_\text{pos}^\text{fixed}$ and $\Phi_\text{pos}^\text{opt}$ that position the robot with respect to a target point in $(x,y,z)$ coordinates, and the distributed task $\Phi_\text{cen}$ that moves the average position of the entire backbone to a target.}
From automatic differentiation, we obtain the gradient of the operator network $\mathcal{G}_\Lambda$ with respect to the input $q_\text{a}$.

\rev{For the positioning tasks, the task gradients are} Dirac at $\bar s$ or $s_*$, \rev{such that} the Jacobian of the composition is the evaluation of the operator network gradient at the required spatial coordinate $\bar s$ or $s_*$ according to~\eqref{eq:jac-comp-pos-fixed},~\eqref{eq:jac-comp-pos-opt}, namely 
\begin{align}
  J_{\Phi_\text{pos}^\text{fixed}\circ \mathcal{G}_\Lambda}(q_\text{a}) &= \frac{\partial \mathcal{G}_\Lambda(q_\text{a})}{\partial q_{\text{a}}}(\bar s) \in \mathbb{R}^{3 \times 3},\\
    J_{\Phi_\text{pos}^\text{opt}\circ \mathcal{G}_\Lambda} (q_\text{a}) &= \frac{\partial \mathcal{G}_\Lambda(q_\text{a})}{\partial q_{\text{a}}}(s_*) \in \mathbb{R}^{3 \times 3}.
\end{align}
Reaching a point translates to a desired task variable $\bar x = \Phi_\text{pos}(\bar r) = [0,0,0]$.
%
\rev{In contrast, for the distributed task, the centroid gradient is distributed along the body, and the composed Jacobian is the average of the nodal shape Jacobians
\begin{align}
  J_{\Phi_{\text{cen}}\circ \mathcal{G}_\Lambda}(q_\text{a})
     = \frac{1}{W}\int_0^L \frac{\partial \mathcal{G}_\Lambda(q_\text{a})}{\partial q_\text{a}}(s)\,ds 
  \in \mathbb{R}^{3\times 3}.
\end{align}
The desired task variable $\bar x = \Phi_\text{cen}(\bar r) = x_0$ is the centroid position. In all three cases, $\Phi_\text{pos}^\text{fixed}$, $\Phi_\text{pos}^\text{opt}$, and $\Phi_\text{cen}$, the control law follows as as~\eqref{eq:clik-infinite} with the respective $\bar x$.
}

Figure~\ref{fig:neural-clik} shows simulations of the neural version of the infinite-dimensional CLIK algorithm for the \rev{three distinct tasks}. 
\rev{In all validation cases, the task is specified by a single target point $x_0\in\mathbb{R}^3$; no full desired shape is prescribed.}
\rev{The control successfully drives the task error to zero while producing smooth robot deformations consistent for the target specification. The distributed task produces a distinct whole-body deformation, where the robot appears to curl around the target as the body average approaches the desired centroid.}

The simulation is for $t\in [0,1]$ with a time step of $dt=0.001$. The weighting matrix is $K=8\cdot\text{diag}(1,1,1)$, i.e., all coordinates are equally weighted, and the magnitude is a balance between convergence speed and overshooting. We use a spatial discretization with $n_\text{s}=100$ points in this example, but this discretization could be arbitrarily finer.
The \emph{Julia}-generated dataset is available online, along with a code example on \emph{GitHub} (\texttt{cayobro/soft-robot-CLIK-neural-operators}).

\section{Conclusion}\label{sec:conclusion}
In this paper, we extend classical closed-loop inverse kinematics (CLIK) to soft robots with infinite-dimensional shape spaces, while keeping the controller finite-dimensional through an end-to-end actuation-to-task map. 
\rev{This approach avoids intermediate approximations and allows to reason on the entire shape during tasks}. By composing an actuation-to-shape model with a shape-to-task map and applying an infinite-dimensional chain rule, we derive the corresponding differential kinematics and obtain a Jacobian-based CLIK algorithm. To make this formulation practical, we learn the actuation-to-shape map from simulation data using a differentiable neural operator network.
\rev{The presented simulations validate the learned continuous interface and show that the framework can handle both point-wise and distributed whole-body objectives. The current study remains a feasibility demonstration: we do not provide an exhaustive comparison against reduced-order baselines, nor do we claim sim-to-real robustness under model mismatch. These aspects, together with more challenging distributed tasks and experimental validation, are important future directions.}


\begin{table}[h]
    \centering
    \caption{\rev{Notation overview.}}
    \label{tab:notation}
    \begin{tabular}{c|l}
    \toprule
    Symbol & Meaning \\
    \midrule
        $m$ & Degrees of actuation \\
        $p$ & Task dimension \\
        $X$ & Hilbert space of robot shapes \\
        $q_\text{a}\in\mathbb{R}^m$ & Actuation coordinates \\
        $x\in\mathbb{R}^p$ & Task coordinates \\
        $r(s)\in X$ & Robot shape \\
        $s\in[0,1]$ & Normalized spatial coordinate \\
        $\Lambda:\mathbb{R}^m\to X$ & Actuation-to-shape map \\
        $\Phi:X\to\mathbb{R}^p$ & Shape-to-task map \\
        $G$ & Learned actuation-to-shape operator \\
        $J_{\Phi\circ\Lambda}$ & End-to-end Jacobian \\
        $K\in\mathbb{R}^{m\times m}$ & Gain matrix \\
        $\bar x$ & Desired task value \\
        $s_*$ & Optimized closest-point coordinate \\
        $\bar s$ & Fixed spatial coordinate \\
        \bottomrule
    \end{tabular}
\end{table}

\bibliography{main.bib}
\bibliographystyle{IEEEtran}

\end{document}